\documentclass[lettersize,journal]{IEEEtran}
\usepackage{amsthm}
\usepackage{amsmath,amsfonts}
\usepackage{algorithmic}
\usepackage{algorithm}
\usepackage{array}
\usepackage[caption=false,font=normalsize,labelfont=sf,textfont=sf]{subfig}
\usepackage{textcomp}
\usepackage{stfloats}
\usepackage{url}
\usepackage{verbatim}
\usepackage{graphicx}
\usepackage{bm}
\usepackage{accents}
\usepackage{xcolor}
\usepackage{wrapfig}
\usepackage{todonotes}
\usepackage{color}
\usepackage[utf8]{inputenc}
\usepackage[T1]{fontenc}
\usepackage{hyperref}
\usepackage{url}
\usepackage{booktabs}
\usepackage{amsfonts}
\usepackage{nicefrac}
\usepackage{microtype}
\usepackage{hhline}
\usepackage{multirow}
\usepackage{soul}
\usepackage[shortlabels]{enumitem}

\usepackage[numbers]{natbib}
\newcommand{\commentout}[1]{}

\newcommand{\LL}[1]{{\color{blue}[LL: #1]}}
\newcommand{\ap}[1]{{\color{orange}[AP: #1]}}
\newcommand{\MW}[1]{{\color{red}[MW: #1]}}

\newcommand{\x}{\mathbf{x}}
\newcommand{\w}{\mathbf{w}}
\DeclareMathOperator*{\argmax}{arg\,max}

\newtheorem{definition}{Definition}
\newtheorem{lemma}{Lemma}
\newtheorem{corollary}{Corollary}
\newtheorem{theorem}{Theorem}
\newtheorem{proposition}{Proposition}
\newtheorem{remark}{Remark}

\DeclareMathOperator\erf{erf}

\begin{document}

\title{On the Robustness of Bayesian Neural Networks\\ to Adversarial Attacks}

\author{Luca Bortolussi,\thanks{Luca Bortolussi with the Department of Mathematics, Informatics and Geosciences, University of Trieste, Trieste, Italy.}
Ginevra Carbone,\thanks{Ginevra Carbone with the Department of Mathematics and Geosciences, University of Trieste, Trieste, Italy.}
Luca Laurenti,\thanks{Luca Laurenti is with Delft Center of Systems and Control, TU Delft University, The Netherlands.}
Andrea Patane,\thanks{Andrea Patane with the School of Computer Science and Statistics, Trinity College, Dublin, Ireland.}
Guido Sanguinetti,\thanks{Guido Sanguinetti is with SISSA, International School of Advanced Studies, Trieste, Italy and with the School of Informatics, University of Edinburgh, Edinburgh, United Kingdom.}
Matthew Wicker\thanks{Matthew Wicker is with the Department of Computer Science, 
University of Oxford, Oxford, United Kingdom.}}



\maketitle

\begin{abstract}
Vulnerability to adversarial attacks is one of the principal hurdles to the adoption of deep learning in safety-critical applications. Despite significant efforts, both practical and theoretical, training deep learning models robust to adversarial attacks is still an open problem. In this paper, we analyse the geometry of adversarial attacks in the over-parameterised limit for Bayesian Neural Networks (BNNs). We show that, in the limit, vulnerability to gradient-based attacks arises as a result of degeneracy in the data distribution, i.e., when the data lies on a lower-dimensional submanifold of the ambient space. As a direct consequence, we demonstrate that in this limit BNN posteriors are robust to gradient-based adversarial attacks. Crucially, by relying on the convergence of infinitely-wide BNNs to Gaussian Processes (GPs), we prove that, under certain relatively mild assumptions, the expected gradient of the loss with respect to the BNN posterior distribution is vanishing, even when each neural network sampled from the BNN posterior is vulnerable to gradient-based attacks. Experimental results on the MNIST, Fashion MNIST, and a synthetic dataset with BNNs trained with Hamiltonian Monte Carlo and Variational Inference, support this line of arguments, empirically showing that BNNs can display both high accuracy on clean data and robustness to both gradient-based and gradient-free adversarial attacks. 
\end{abstract}


\section{Introduction}

Adversarial attacks are small, potentially imperceptible, perturbations of test inputs that can lead to catastrophic misclassifications in high-dimensional classifiers such as deep Neural Networks (NN).
Since the seminal work of \citet{szegedy2013intriguing} adversarial attacks have been intensively studied and even highly accurate state-of-the-art deep learning models, trained on very large data sets, have been shown to be susceptible to such attacks \citep{goodfellow2014explaining,zhang2019adversarial}. In the absence of effective defenses, the widespread existence of adversarial examples has raised serious concerns about the security and robustness of models learned from data \citep{biggio2018wild,zhuo2023security}.
As a consequence, the development of machine learning models that are robust to adversarial perturbations is an essential pre-condition for their application in safety-critical scenarios, such as autonomous driving, where model failures can lead to fatal or costly accidents.

Many adversarial attack strategies are based on identifying directions of high variability in the loss function by evaluating the gradient w.r.t.\ to the neural network input \citep{goodfellow2014explaining,pgdattack}. Since such variability can be intuitively linked to uncertainty in the prediction, Bayesian Neural Networks (BNNs) \citep{neal2012bayesian,jia2023energy,li2020continual,chien2015bayesian} have been recently suggested as a more robust deep learning paradigm, a claim that has also found empirical support  \citep{feinman2017detecting,wicker2021bayesian,bekasov2018bayesian,liu2018adv,yuan2020gradient}. However, neither the source of this robustness, nor its general applicability are well understood mathematically.

In this paper we show a remarkable property of BNNs: in a suitably defined  limit, we prove that the gradients of the expected loss function of an infinitely-wide BNN w.r.t.\ the input vanish. 
Our analysis shows that adversarial attacks for highly accurate NNs arise from the low dimensional support of the data generating distribution. By averaging over nuisance dimensions, {under certain assumptions on the geometry of the space where the data come from, assumed to be a manifold,} BNNs achieve zero expected gradient of the loss and are thus,  provably immune to gradient-based adversarial attacks. Specifically, we first show that, for any neural network achieving zero loss, adversarial attacks arise in directions orthogonal to the data manifold. Then, we rely on the submanifold extension lemma \citep{leeintroduction} to show that in the limit of infinitely-wide layers, for any neural network and any weights set there exists another weights set (of the same neural network architecture) achieving the same loss and with opposite loss gradients orthogonal to the data manifold on a given point. 
Finally, {  by relying on the convergence of BNNs to Gaussian Processes (GPs) \citep{neal2012bayesian} and under the assumption that the data manifold is a subspace}, we show that for infinitely-wide BNNs  the expectation of the gradient w.r.t.\ the posterior distribution in a direction orthogonal to the data manifold  vanishes. Crucially, our results guarantees that, in the limit, BNNs' posteriors are provably robust to gradient-based adversarial attacks even when neural networks sampled from the posterior are vulnerable to such attacks. 

We experimentally support our theoretical findings on various BNN architectures trained with Hamiltonian Monte Carlo (HMC) and with Variational Inference (VI){\color{blue}, specifically Bayes by Backprop \citep{blundell2015weight},} on MNIST, Fashion MNIST and the half moons datasets, empirically showing that the magnitude of the gradients decreases as more samples are taken from the BNN posterior. 
We then explore the robustness of BNNs to adversarial attacks {experimentally in these settings}. 
In particular, we conduct a large-scale experiment on thousands of different neural networks, empirically finding that, in the cases here analysed, for BNNs higher accuracy tend to correlate with higher robustness to gradient-based adversarial attacks, contrary to what observed for deterministic NNs trained via standard Stochastic Gradient Descent (SGD). Finally, we also investigate the robustness of BNNs to gradient-free adversarial attacks, empirically showing that BNNs are substantially more robust than their deterministic counterpart even in this setting.

In summary, this paper makes the following contributions:
\begin{itemize}
    \itemsep0em 
    \item A proof that, in the infinitely-wide layers and large data limit setting, the gradient of the loss function w.r.t.\ the input only preserves the component which is orthogonal to the data manifold (Section \ref{sec:GradAdvAttacks}) and that for any weights set of a neural network there exists  another weight set with same loss and opposite orthogonal gradients (Section \ref{sec:AdvRobBNNs}). 
    \item A proof that for Gaussian Processes (GPs), and consequently for infinitely-wide limit BNNs, the expected posterior gradient of the loss vanishes when projected in a direction orthogonal to the data manifold, thus providing robustness to BNNs (Section \ref{sec:adv_robustness_bayesian_averaging}).
    \item Experiments showing empirically that BNNs are more robust to both  gradient-based and gradient-free attacks than their deterministic counterpart and can resist the well known accuracy-robustness trade-off (Section \ref{sec:empirical_results}).\footnote{The code for the experiments can be found at 
\url{https://github.com/matthewwicker/OnTheRobustnessOfBNNs}.
    } 
\end{itemize}
A preliminary version of this work appeared in \cite{carbone2020robustness}. This work extends \cite{carbone2020robustness} in several aspects. In \cite{carbone2020robustness} we proved that given an infinitely-wide neural network with zero loss and a non-zero orthogonal gradient to the data manifold, there exists another zero-loss neural network with opposite orthogonal gradient and use this result to conjecture that under certain conditions BNNs may achieve zero gradients in expectation . In this paper, {  in Section \ref{sec:adv_robustness_bayesian_averaging}, this conjecture is shown to be true and proved explicitly by relying on the convergence of BNNs and Gaussian processes (GPs).} Furthermore, we substantially extend the discussion and the theoretical analysis,  and improve the empirical results with gradient-free adversarial attacks {  and a comparison between the robustness of GPs and BNNs}.

The paper is structured as follows. In Section \ref{sec:background} we introduce background on infinitely-wide neural networks and BNNs. In Section \ref{sec:GradAdvAttacks} we will first show that for highly accurate neural networks the gradient of the loss is non-zero only in directions orthogonal to the data manifold. Then, in Section \ref{sec:AdvRobBNNs} we will prove that for any neural network and weight set there exists another weight set of the same neural network with same loss and opposite orthogonal gradients to the data manifold. By averaging over these weight sets and relying on the convergence of BNNs to Gaussian processes (GPs), in Section \ref{sec:adv_robustness_bayesian_averaging} we prove that for a BNN that achieves zero loss on the data manifold the expected gradient of the loss is zero, thus making them robust to adversarial attacks. Section \ref{sec:limitations} discusses consequences and limitations of our results.
Empirical results in Section \ref{sec:empirical_results} will support our theoretical findings.

\subsection{Related Work}

{ Adversarial attacks for deterministic neural networks have been the subject of extensive analyses \cite{liu2021training,wiyatno2019adversarial,yuan2019adversarial,gallego2020incremental,panda2020quanos,ilyas2019adversarial,lin2019defensive}, which have led to the development of multiple defence and attack methods over the recent years \cite{biggio2018wild,liu2021trainingDet,wang2023generating}. 
Adversarial examples have been found to be so widespread in state-of-the-art deterministic NNs that they have even been hypothesized to be an intrinsic property of certain models \cite{ilyas2019adversarial} or datasets \cite{tsipras2018robustness}.} 
{ Interestingly, early experimental results with BNNs suggested a diametrically different behaviour than that of their deterministic counter-part.
In fact, empirical observations on the increased adversarial robustness of BNNs have been made in various works both against gradient-based adversarial attacks \citep{pang2021evaluating,smith2018understanding,uchendu2021robustness} and gradient-free adversarial attacks \citep{yuan2020gradient} as well as on reinforcement learning settings \citep{michelmore2019uncertainty}{   and more recently also on relatively large convolutional neural network architectures \cite{pang2021evaluating}}. }However, while these works present empirical evidences on the robustness of BNNs, they do not give any theoretical justification on the mechanisms that lead to BNN robustness. First attempts to understand the robustness properties of BNNs have been considered in \cite{bekasov2018bayesian,gal2018sufficient}.
In particular, \cite{bekasov2018bayesian} defined Bayesian adversarial spheres and empirically showed that, for BNNs trained with HMC, adversarial examples tend to have high uncertainty. Instead \cite{gal2018sufficient} derived sufficient conditions for idealised BNNs to avoid adversarial examples. However, it is unclear how such conditions could be checked in practice, as it would require one to check that the BNN architecture is invariant under all the symmetries of the data. 

Because of the capabilities of BNNs to model { epistemic} uncertainty, which can be intuitively linked to their robustness properties, various approaches have been proposed to detect adversarial examples for BNNs. {  \citep{feinman2017detecting,rawat2017adversarial} propose to use the uncertainty on the predictions of a BNN as a way to flag adversarial attacks. However, such methods have been shown to be easily fooled by appropriately crafted adversarial attacks \cite{carlini2017adversarial,grosse2018limitations}. Consequently, }
formal verification methods \citep{wicker2020probabilistic,berrada2021verifying} to detect adversarial examples for BNNs have been introduced. These methods have been followed by techniques to perform adversarial training for BNNs \citep{zhang2021robust,ye2018bayesian,liu2018adv,wicker2021bayesian}, where additional robustness constraints or penalties are considered directly at training time. Interestingly, empirical results obtained with such techniques, highlighted how, in the Bayesian settings, high accuracy and high robustness often are positively correlated with each other. The theoretical framework we develop in this paper further confirms and grounds these findings.




\section{Background}
\label{sec:background}


Let $f^{true}:\mathcal{M}\to\mathbb{R}$ be a function defined on a data manifold $\mathcal{M}\subseteq X \subseteq \mathbb{R}^d$ with $X$ being the ambient (or embedding) space.\footnote{For simplicity of presentation, we assume a scalar output. The results of this paper naturally extend to the multi-output case by treating each output component  similarly to the single output case.} 
We consider the problem of approximating $f^{true}$ via the learning of an 
$M+1$ layers neural network $f(\cdot,\w)$,  with $\w \in \mathbb{R}^{n_\w}$ being the aggregate vector of weights and biases. 
Formally, for $\x=(x_1,...,x_d) \in X$, $f(\x,\w)$ is defined iteratively over the number of layers as:
\begin{align}
    f^{(1)}_i(\x)&= \sum_{j=1}^d w^{(1)}_{ij}x_j+b^{(1)}_i \label{eq:nn1} \\
    f^{(m)}_i(\x)&=  \sum_{j=1}^{n_{m-1}} w^{(m)}_{ij} \phi(f^{(m-1)}_j(\x))+b^{(m)}_i, \label{eq:nn2} \\
    f(\x,\w) &= f^{(M+1)}(\x), \label{eq:nn3}
\end{align}
$\textrm{for} \; m=2,\ldots,M+1$, where $n_m$ is the number of neurons in the $m$-th layer and $\phi$ is the activation function – which we assume to be continuous and with bounded derivatives.
In order to learn the weights of $f(\x,\w)$ one considers a dataset $D_N$ composed of $N$ points, 
$D_N = \{ (\x_i,y_i) \; |\; \x_i \in \mathcal{M},\, y_i \in \mathbb{R},\, i=1,\ldots,N \}$.
In a frequentist fashion, one can then use the dataset to quantify the distance between $f^{true}$ and $f(\cdot,\w)$ by evaluating it on a loss function $L(\x,\w)$  of the form $L(\x,\w) = \ell( f(\x,\w), f^{true}(\x))$, with $\ell(\cdot,\cdot)$  chosen accordingly to the semantic of the problem at hand (e.g., square loss or cross-entropy).\footnote{For simplicity of notation we omit the explicit dependence on the true function from the loss.}  
Intuitively, minimisation of the loss function over the weight vector $\w$ leads to increasing fit of $f(\x,\w)$  to $f^{true}(\x)$, with zero-loss indicating that the fit is exact {on $D_N$}. 

In this paper, we aim at analysing the adversarial robustness of $f(\x,\w)$. In order to do so, we will rely on crucial results from Bayesian learning and on the properties of infinitely-wide neural networks. The remainder of this section is dedicated to the review of such notions.


\subsection{Infinitely-Wide Neural Networks}
In our analysis we will rely on the notion of infinitely-wide NNs, i.e. NNs with an infinite number of neurons. 
 
 
\label{sec:overparameterizeD_Nns}


\begin{definition}[Infinitely-wide neural network]\label{def:infiniteNN}
Consider a family of neural networks $\{f(\x,\w_{n_\w})\}_{n_\w>0}$ 
of Equations~\eqref{eq:nn1}--\eqref{eq:nn3}, with a fixed number of neurons for $m=1,\ldots,M-1$ and a variable number of neurons $n_M$ in the last hidden layer. We say that 
\begin{equation}
    f^\infty(\x) := \lim_{n_M \to \infty} f(\x,\w_{n_\w}) \quad \forall \x \in X, \label{eq:infiniteNN}
\end{equation}
is an infinitely wide neural network if the limit above exists and if the resulting function defines a mapping from $X$ to $\mathbb{R}$.
Furthermore, we call $\mathcal{F}$ the set of such limit functions.
\end{definition}
%
The interest behind the set of infinitely-wide neural networks lies in the fact that they are universal approximators \citep{cybenko1989approximation,hornik1991approximation}.\footnote{Notice that the limit in Definition \ref{def:infiniteNN} is taken only w.r.t.\ the last hidden layer. Similar results, albeit with additional care needed for the definition of the limiting sequence, can be obtained by taking the limit w.r.t.\ all the hidden layers \citep{matthews2018gaussian}.} More precisely, under the assumption that the true function $f^{true}$ is continuous, we have that:
\begin{align}\label{eq:universal_theorem}\nonumber
    \forall \epsilon > 0, \; \exists f^* \in \mathcal{F} &\; \textrm{s.t.} \; \forall x \in \mathcal{M},\\ &\; | f^{true}(x) - f^*(x) | < \epsilon.
\end{align}
That is, $\mathcal{F}$ is dense in the space of continuous functions. 
Furthermore, it is possible to show that any smooth function with bounded derivatives can be represented exactly by an infinitely wide NN with bounded weights norm (i.e.\ with bounded sum of the squared Euclidean norm of the weights in the network) \citep{ongie2020function}. 
We will rely on these crucial properties of infinitely-wide neural networks to reason about their behaviour against adversarial attacks. 
We remark that the existence of such a neural network $f^*$ approximating the true underlying function does not necessarily mean that it will be automatically found through gradient descent-based training on a finite dataset. However, 
recent results \citep{rotskoff2018neural} have shown that, under mild conditions, the loss function, as a functional over the distribution over weights of an infinitely-wide NN, is a convex functional, and hence the gradient flow  of stochastic gradient descent will converge to a unique distribution over weights.
%
That is, given an infinitely-wide NN $f^\infty$ and a sequence of datasets $\{D_N\}_{N>0}$ of cardinality $N$ extracted from the data manifold $\mathcal{M}$, we have that:
\begin{align}\label{eq:zero_loss_regime}
    \lim_{N \to \infty} \ell\big( f_{D_N}^\infty(\x), f^{true}(\x)\big) = 0, \quad \forall \x \in \mathcal{M}
\end{align}
where $f^\infty_{D_N}$ represents the infinitely-wide NN trained on $D_N$ until convergence.
In Section \ref{sec:GradAdvAttacks} we will show how infinitely-wide deterministic neural networks, i.e., neural networks where each weight or bias is a scalar, can be vulnerable to adversarial attacks even when the loss is zero, while infinitely-wide Bayesian neural networks, under certain assumptions on the geometry of the  data manifold, are provably robust to gradient-based adversarial attacks.

\subsection{Bayesian Neural Networks}
Bayesian modelling aims to capture the uncertainty of data driven models by defining ensembles of predictors \citep{barber2012bayesian}; it does so by turning model parameters 
into random variables. In the NN scenario, 
one starts by putting a prior measure over the network weights $p(\mathbf{w})$ \citep{neal2012bayesian}.\footnote{In the remainder of this paper, we employ the common notation of indicating density functions with $p$ and their corresponding probability measures with $P$.} 
The fit of the network with weights $\mathbf{w}$ to the data $D$ is assessed through the likelihood $p(D\vert\mathbf{w})$
\citep{bishop}.\footnote{Notice that in the Bayesian setting the likelihood is a transformation of the loss function used in deterministic settings. In the rest of the paper we use both terminologies, and the loss is not to be confused with that used in Bayesian decision theory \citep{bishop}.} Bayesian inference then combines likelihood and prior via the Bayes theorem to obtain a {\it posterior} distribution over the NN parameters 
\begin{align}\label{eq:Bayes}
p\left(\mathbf{w}\vert D\right)\propto  p\left(D\vert\mathbf{w}\right)p\left(\mathbf{w}\right).
\end{align}
Unfortunately, it is in general infeasible to compute the posterior distribution exactly for non-linear/non-conjugate models such as deep NNs, so that approximate Bayesian inference methods are employed in practice. 
%
Asymptotically exact samples from the posterior distribution can be obtained via procedures such as Hamiltonian Monte Carlo (HMC) \citep{neal2011mcmc}, while approximate samples can be obtained more cheaply via Variational Inference (VI) \citep{blundell2015weight}. 
Irrespective of the posterior inference method of choice, Bayesian empirical predictions at a new input $\mathbf{x}$ are obtained from an ensemble of $n$ NNs, each with its individual weights drawn from the posterior distribution $p(\mathbf{w}|D)$: 
\begin{equation}\begin{split}
     \langle f(\mathbf{x},\mathbf{w})\rangle_{p\left(\mathbf{w}\vert D\right)}
    \simeq \frac{1}{n}\sum_{i=1}^n f(\mathbf{x},\mathbf{w}_i)
\end{split}\label{predDist}\end{equation} 
where $\mathbf{w}_i\sim p\left(\mathbf{w}\vert D\right)$ and $\langle\cdot\rangle_{p\left(\mathbf{w}\vert D\right)}$ denotes expectation w.r.t.\ the posterior distribution $p\left(\mathbf{w}\vert D\right)$. 


{ Note that the definition of a distribution over the weights $p(\mathbf{w})$ naturally leads to the definition of a probability measure over the set of continuous functions $f:\mathbb{R}^d \to \mathbb{R}$ that can be represented by the neural network. In particular, as common in the literature \cite{adler2010geometry,billingsley2013convergence}, we consider the probability measure $P$ generated by
the finite-dimensional distributions of the BNN, i.e., the joint distributions of $f(\x_1, {\mathbf{w}}), ..., f(\x_k, {\mathbf{w}}),$
where $k$ is an arbitrary integer and $\x_1,...\x_k \in X$.  We stress that in general not all path properties, such as continuity or differentiability, can be determined using the finite dimensional distributions. However, as the neural networks architectures considered in this paper are continuous by assumption, without any lost of generality and as common in the literature \citep{adler2010geometry}, we assume that $f(\cdot,\mathbf{w})$ is separable, i.e., countable dense subsets of input points suffice to determine the properties of $f(\cdot,\mathbf{w})$.  } 

{ 

In this paper, as also common in the literature \citep{blundell2015weight}, we will consider Gaussian priors $p(\mathbf{w})$.
The following result shows how an independent Gaussian prior over the parameters of an infinitely-wide BNN induces a Gaussian prior over the space of functions. 
\begin{proposition}[\citep{neal2012bayesian,matthews2018gaussian}]
\label{prop:PriorsWeighttoFunction}
Consider the following neural network
$f(\x,\w)$ with a single hidden-layer defined as 
\begin{align}
    f^{(1)}_i(\x)&= \sum_{j=1}^d w^{(1)}_{ij}x_j+b^{(1)}_i  \\
    f^{(2)}_i(\x)&=  \sum_{j=1}^{{n_1}} w^{(2)}_{ij} \phi(f^{(1)}_j(\x))+b^{(2)}_i,  \\
    f(\x,\w) &= f^{(2)}(\x).
\end{align}
Assume that to each weight and bias are associated independent normal priors such that  $w^{(1)}_{ij} \sim \mathcal{N}(0,\frac{\sigma_w^2}{d}),$ $w^{(2)}_{ij} \sim \mathcal{N}(0,\frac{\sigma_w^2}{n_1}),$ $b^{(1)}_i,b^{(2)}_i \sim \mathcal{N}(0,{\sigma_b^2}).$
Then, for ${n_1}\to\infty$, the prior on $f_{i}(\x,\w)$ converges in distribution to a Gaussian process (GP) with zero mean and  covariance function $K(\x,\x')=\sigma_b^2 + \sigma_w^2 C(\x,\x'), $
where $C(\x,\x')$ is a function dependent on the BNN architecture.
\end{proposition}
Note that, while Proposition \ref{prop:PriorsWeighttoFunction} is stated only for BNNs with one hidden layer, analogous results can be derived for the multiple hidden layer case and also for more complex architectures such as convolutional neural networks \citep{matthews2018gaussian,lee2017deep,garriga2018deep}. 
 We stress that in the case of multiple hidden layers, care must be taken in how the size of the various layers goes to infinity to guarantee convergence.
In what follows, we will simply assume that for an infinitely-wide BNN, the conditions for convergence are always satisfied.

Thanks to Proposition \ref{prop:PriorsWeighttoFunction} we have that an infinitely-wide BNN is equivalent to a GP. This allows us to use the favourable analytical properties of GPs to study BNN robustness in the limit. In particular, in Section \ref{sec:adv_robustness_bayesian_averaging} we will rely on the fact that the derivative\footnote{In this paper we will always focus on mean-square derivatives \citep{adler2010geometry} for probabilistic models. Hence, in what follows, we will simply refer to them as derivatives.} of a GP is still a GP with a kernel given by the derivative of the kernel of the original GP \citep{adler2010geometry}. This is a key result that we will use to show how the orthogonal gradient of the loss for trained BNNs vanishes along the data manifold.

}

\subsection{Adversarial Attacks for Bayesian Neural Networks}
\label{sec:posterior_pred_attack}
Given an input point $\mathbf{x}\in \mathcal{M}$ and a strength (i.e.\ maximum perturbation magnitude) $\epsilon>0$, the worst-case adversarial perturbation can be defined as the point $\tilde{\mathbf{x}}$ in the $\epsilon$-neighbourhood around $\mathbf{x}$ that maximises the  loss function  $L$:
$$\tilde{\mathbf{x}} := \argmax_{ \tilde{\mathbf{x}}:||\tilde{\mathbf{x}}-\mathbf{x}|| \leq \epsilon}\langle L (\tilde{\mathbf{x}},\mathbf{w})\rangle_{p\left(\mathbf{w}\vert D\right)}.$$
If the network prediction on $\tilde{\mathbf{x}}$ differs from the original prediction on $\mathbf{x}$, then we call $\tilde{\mathbf{x}}$ an \emph{adversarial example}.
As $f(\mathbf{x},\mathbf{w})$ is non-convex, computing $\tilde{\mathbf{x}}$ is a non-convex optimisation problem for which several approximate solution methods have been proposed. 
In this paper we will primarily focus on what arguably is the most commonly employed class among them, i.e., gradient-based attacks, that is attacks that employ the loss function gradient w.r.t.~$\x$ to maximise the loss \citep{biggio2018wild}.
One such attacks is the Fast Gradient Sign Method (FGSM) \citep{goodfellow2014explaining}  which works by approximating $\tilde{\mathbf{x}}$ by taking an  $\epsilon$-step in the direction of the sign of the gradient at $\mathbf{x}$. 
 In the context of BNNs, where attacks are  against the posterior distribution, applying FGSM yields 
\begin{align}
\tilde{\mathbf{x}} &= \mathbf{x}+\epsilon\,\mathrm{sgn}\;\langle\nabla_{\mathbf{x}}L(\mathbf{x},\mathbf{w})\rangle_{p\left(\mathbf{w}\vert D\right)}\label{BayesFGSM}\\ &\simeq\mathbf{x}+\epsilon\,\mathrm{sgn}\left(\sum_{i=1}^n\nabla_{\mathbf{x}}L(\mathbf{x},\mathbf{w}_i)\right)
\end{align}
where the final expression is a Monte Carlo approximation with $n$ samples $\mathbf{w}_i$ drawn from the posterior $p(\mathbf{w}|D)$. 
Other gradient-based attacks, as for example Projected Gradient Descent method (PGD)  \citep{pgdattack}, modify FGSM by taking consecutive gradient iterations or by scaling the attack by the magnitude of the gradient.
Crucially, however, they all rely on the gradient vector to guide the attack. 

In the following sections we will rely on the fact that the expected loss gradient in Equation \eqref{BayesFGSM} can be decomposed into its projection into a direction parallel to the data manifold and into a direction orthogonal to the data manifold. The parallel expected loss gradient naturally vanishes for very accurate neural neural networks as the loss will tend to be zero everywhere in the data manifold. Instead, the orthogonal projection will in general not be zero. However, perhaps surprisingly, we will show that for infinitely-wide Bayesian neural network also the orthogonal expected loss gradient vanishes. 

Before considering results specific to BNNs in Section \ref{sec:adv_robustness_bayesian_averaging}, in Section \ref{sec:GradAdvAttacks} we will focus on results that hold for both deterministic and Bayesian neural networks  (Lemma \ref{trivialLemma}) and that are specific to deterministic NNs (Proposition \ref{Proposition:ExtensionLemmaForNNs}), showing how these can be vulnerable to adversarial attacks even when they learn the true function perfectly.

\section{Gradient-Based Adversarial Attacks for Neural Networks}
\label{sec:GradAdvAttacks}


Equation \eqref{BayesFGSM} suggests a possible mechanism  through which BNNs might acquire robustness against adversarial attacks: averaging under the posterior might lead to cancellations in the expectation of the gradient of the loss. It turns out that this averaging property is intimately related to the geometry of the data manifold $\mathcal{M}$. As a consequence, in order to study the expectation of the gradient of the loss for BNNs, we first introduce  results that link the geometry of $\mathcal{M}$ to adversarial attacks.


We start with a trivial, yet important result, which holds for any neural network, both Bayesian and deterministic: for a NN that achieves zero loss on the whole data manifold $\mathcal{M}$, the loss gradient is constant (and zero) along the data manifold for any $\x\in\mathcal{M}$. Therefore,  
in order to have adversarial examples the dimension of the data manifold $\mathcal{M}$ must necessarily be smaller than the dimension of the ambient space, that is, $\mathrm{dim}\left(\mathcal{M}\right)<\mathrm{dim}\left(X\right)=d,$ where $\mathrm{dim}\left(\mathcal{M}\right)$ denotes the dimension of $\mathcal{M}$.

\begin{lemma}\label{trivialLemma}
 Assume that $\mathcal{M}$ is a smooth closed manifold and that  $\forall \x \in \mathcal{M}$ $L(\mathbf{x},\w)=0$, that is  $f(\x,\w)$ achieves zero loss on $\mathcal{M}$. Then, if $f$ is vulnerable to gradient-based attacks at $\x^{*}\in\mathcal{M}$, $\mathrm{dim}\left(\mathcal{M}\right)<\mathrm{dim}\left(X \right)$ in a neighbourhood of $\x^*$, i.e. $\mathcal{M}$ is locally homeomorphic to a space of dimension smaller than the ambient space $X$. 
\end{lemma}
\begin{proof}
By assumption $\forall \x\in\mathcal{M}, L(\x,\w)=0$, which implies that the gradient of the loss is zero along the data manifold.  However, if $f$ is vulnerable to gradient based attacks at $\x^*$ then the gradient of the loss at $\x^*$ must be non-zero. Hence, there exists an open neighbourhood $\mathcal{B}$ of $\x^*$ such that $\mathcal{B}\not\subseteq\mathcal{M}$, which implies $\mathrm{dim}(\mathcal{M})<\mathrm{dim}\left(X \right)$ locally around $\x^*$. 
\end{proof}
Lemma \ref{trivialLemma} confirms the widely held conjecture that adversarial attacks may originate from degeneracies of the data manifold \citep{goodfellow2014explaining,fawzi2018adversarial}. In fact, it has been already empirically noticed \citep{khuory} that adversarial perturbations often arise in directions  normal to the data manifold. 
The suggestion that lower-dimensional data structures might be ubiquitous in NN problems is also corroborated by recent results \citep{goldt2019modelling} showing that the characteristic training dynamics of NNs are intimately linked to data lying on a lower-dimensional manifold. Notice that the implication is only one way; { it is perfectly possible for the data manifold to be low dimensional and still not vulnerable at many points. Consequently, the fact that the data manifold has a smaller dimension than the ambient space is a necessary, but not sufficient, condition for vulnerability to adversarial attacks.} 

We note that, as discussed in Section \ref{sec:overparameterizeD_Nns}, at convergence of the training algorithm and in the limit of infinitely-many data, infinitely-wide neural networks are guaranteed to achieve zero loss on the data manifold, satisfying the assumption of Lemma \ref{trivialLemma}. As a result, once an infinitely-wide NN is fully trained, for any $\x\in\mathcal{M}$ the gradient of the loss function is orthogonal to the data manifold as it is  zero along the data manifold, i.e., $\nabla_{\x} L(\x,\w)=\nabla^\perp_\x L(\x,\w)$, where  $\nabla^\perp_\x$  denotes the gradient projected into the normal subspace of $\mathcal{M}$ at $\x$. We stress that for a given NN, $\nabla^\perp_\x L(\x,\w)$ is in general non-zero even if the network achieves zero loss on $\mathcal{M}$ (this is formalized in the next subsection), thus explaining the existence of adversarial examples even for very accurate classifiers. Crucially, in Section \ref{sec:adv_robustness_bayesian_averaging} we show that for BNNs, when averaged w.r.t.\ the posterior distribution, the orthogonal gradient vanishes.

\subsection{A Symmetry Property of Neural Networks}
\label{sec:AdvRobBNNs}


\noindent
Before considering the BNN case, in Proposition \ref{Proposition:ExtensionLemmaForNNs} below we show a symmetry property of neural networks: given a  neural network, we can always find an infinitely-wide NN that has the same loss but opposite orthogonal gradient. 
In order to prove this result, we first introduce Lemma \ref{extension_lemma}, which is a generalization of the submanifold extension lemma  
and a key result we leverage. It proves that any smooth function defined on a submanifold $\mathcal{M}$ can be extended to the ambient space, in such a way that the choice of the derivatives orthogonal to the submanifold is arbitrary. 

\begin{lemma}[\cite{anders2020fairwashing}]\label{extension_lemma}
Assume that $\mathcal{M}$ is a smooth closed manifold. Let $T_\x\mathcal{M}$ be the tangent space of $\mathcal{M}$ at a point $\x\in\mathcal{M}$. Let $V=\sum_{i=\mathrm{dim}\left(\mathcal{M}\right)+1}^d v^i \partial_i$ be a conservative vector field along $\mathcal{M}$ which assigns a vector in $T_\x\mathcal{M}^\perp$
for each $\x\in\mathcal{M}$. For any smooth function $f^{true}:\mathcal{M}\to\mathbb{R}$ there exists a smooth extension $F:X\to\mathbb{R}$ such that 
$$ F|_\mathcal{M} = f^{true},
$$
where $F|_\mathcal{M}$ denotes the restriction of $F$ to the submanifold $\mathcal{M}$, and such that the derivative of the extension $F$ is 
\begin{align*}
    \nabla_\x F(\x)=(\nabla_1 f^{true}(\x), \ldots, &\nabla_{\mathrm{dim}\left(\mathcal{M}\right)} f^{true}(\x),\\ &v^{\mathrm{dim}\left(\mathcal{M}\right)+1}(\x), \ldots, v^{d}(\x))
\end{align*}
for all $\x\in\mathcal{M}$.
\end{lemma}
\noindent
Notice that in Lemma \ref{extension_lemma}, in $\nabla_\x F(\x)$, we pick the local coordinates at $\x\in \mathcal{M}$, such that the first set of components parametrises the data manifold. We stress that, as $\mathcal{M}$ is smooth, this is without any loss of generality \citep{leeintroduction}. 
Lemma \ref{extension_lemma}, together with the universal approximation capabilities of NNs \citep{hornik1991approximation}, is employed in Proposition \ref{Proposition:ExtensionLemmaForNNs} to show that for any possible value $\mathbf{v}$ of the orthogonal gradient to the data manifold in a point,  there exists at least two (possibly not unique) different weight vectors that achieve zero loss and have orthogonal gradients respectively equal to $\mathbf{v}$ and $-\mathbf{v}$.

\begin{proposition}\label{Proposition:ExtensionLemmaForNNs}
 Consider an infinitely-wide NN $f$ with  smooth, bounded, and non-constant activation functions and an input $\x\in \mathcal{M}$, where $\mathcal{M}$ is a smooth closed manifold.  Then, for any smooth function ${f}^{true}:\mathcal{M} \to \mathbb{R}$ and vector $\mathbf{v}\in \mathbb{R}^{\mathrm{dim}\left(X \right)-\mathrm{dim}\left(\mathcal{M}\right)}$, there exist  $\w_1, \w_2$
 such that
 \begin{align}
 f(\cdot,\w_1) |_{\mathcal{M}} = {f}^{true}  =f(\cdot,\w_2) |_{\mathcal{M}} \label{eq:restriction}\\
\nabla_\x^\perp f(\x,\w_1)=\mathbf{v}=-\nabla_\x^\perp f(\x,\w_2).\label{eq:opposite_grads}
\end{align}
\end{proposition}
\begin{proof}
From Lemma \ref{extension_lemma} we know that there exist smooth extensions $F^+$ and $F^-$ of $f^{true}$ 
to the embedding space such that 
$\nabla_\x^\perp F^+(\x) =\mathbf{v}=-\nabla_\x^\perp F^-(\x)$.
As a consequence, to conclude the proof it suffices to apply Theorem 3 in \citep{hornik1991approximation} that guarantees that infinitely-wide neural networks are \emph{uniformly 1-dense} on compacts in $\mathcal{C}^1(X)$, under the assumptions of smooth, bounded, and non-constant activation functions. 
Specifically, for any $F\in \mathcal{C}^1(X)$ and $\epsilon>0$, for any compact $X'\subseteq X$ there exists a set of weights $\w$ s.t.
\begin{align*}
    \max\Big\{ \sup_{\x\in X'}&|| F(\x)-f(\x,\w)||_{\infty},\\ &\sup_{\x\in X'}|| \nabla F(\x) - \nabla f(\x,\w) ||_{\infty} \Big\} \leq \epsilon.
\end{align*}
As $F^+,F^- \in \mathcal{C}^1(X)$, this concludes the proof.

\end{proof}
\noindent
Note that by the chain rule, the gradient of the loss is proportional to the gradient of the NN. As a consequence, 
Proposition \ref{Proposition:ExtensionLemmaForNNs} guarantees that,
for infinitely-wide NNs, for any weights set achieving the minimum loss, 
 there exists another weights set with same loss and opposite orthogonal gradient of the loss w.r.t.\ the input. This has various implications: i) deterministic NNs trained on a data manifold could exhibit arbitrarily large gradients in directions orthogonal to the data manifold, even when they learn the latent function perfectly, ii) in a Bayesian framework, by averaging over weights sets that have opposite orthogonal gradients, one could achieve a robust model that has vanishing expected orthogonal gradient. In the next Section, in Theorems \ref{Th:GaussianLimitCancellationGaussian} and \ref{th:MainTheoremClassification} we show that, under some relatively mild assumptions, this is indeed the case. However, we should already emphasize  that such a result does not hold by simply averaging the set of weights w.r.t.\ any distribution. Intuitively, for this result to hold, it is required that each set of weights achieving a given gradient value has the same measure as the set of weights with the same loss and opposite orthogonal gradient value.


\section{Adversarial Robustness via Bayesian Averaging}
\label{sec:adv_robustness_bayesian_averaging}

{

In order to prove that BNNs have vanishing orthogonal gradients, we start from a general Gaussian process { (GP)} \cite{liu2020gaussian} trained with a given dataset. We show that, under the assumption that the data manifold is a linear subspace of the ambient space, the projection of the expected gradient of the GP in a direction orthogonal to the data manifold vanishes for all points in the data manifold. This shows that GPs are able to obtain perfect cancellation of the orthogonal gradients.
By relying on the convergence of BNNs to GPs we then extend this result to infinitely-wide  BNNs. 

We first state our results for a regression setting, the classification setting will then be considered in Section \ref{Sec:Classification}. { In particular, in Theorem \ref{Th:GaussianLimitCancellationGaussian} below we show that, for a wide class of covariance functions, a GP trained on a regression problem has zero expected orthogonal gradients. In   Corollary \ref{Corollary:GaussianLimitCancellation} we then extend this result to BNNs. 

\begin{theorem}
\label{Th:GaussianLimitCancellationGaussian}
Let $z(\x)$ be a zero-mean Gaussian process with covariance function $K:\mathbb{R}^d \times \mathbb{R}^d \to \mathbb{R}.$ Call $z(\x)\mid D_N$ the posterior GP obtained by training $z$ on a regression problem with data set $D_N=\{ (\x_i,\mathbf{y}_i) \; |\; i=1,\ldots,N \}$ and additive i.i.d. zero mean Gaussian observation noise of variance $\sigma^2\geq 0$.
 Assume that:
    \begin{itemize}
 \item For $\x',\x'' \in \mathcal{M},$ $K(\x',\x'')=g(\sum_{i=1}^d(x_i'-x_i'')^l),$ $l>1$ or $K(\x',\x'')=g(\sum_{i=1}^dx_i' x_i''),$ where $g$ is any twice differentiable function {such that $K$ is a valid kernel}.   
 \item $\mathcal{M}$ is a linear subspace of $X\subseteq \mathbb{R}^d$.
    \end{itemize}
Then, for any $\x\in \mathcal{M}$ it holds that 
    \begin{equation}
       \langle\nabla_\x^\perp z(\x)\rangle_{p( z(\x)\mid D_N)} = 0.
    \label{eq:expected_gradient_GP}
    \end{equation}
\end{theorem}
\begin{proof}
Conditional distributions of jointly Gaussian random variables are still Gaussian \citep{rasmussen2006gaussian}. Consequently,  $ z\mid D_N$ is a Gaussian process with mean at $\x$ given by
\begin{align}
\label{Eqn:PosteriorGPRegression}
\langle z(\x)\rangle_{p( z(\x)\mid D_N)}=K(\x,X)\big( K(X,X) + \sigma^2I \big)^{-1}[\mathbf{y}_1,...,\mathbf{y}_N]^T,   
\end{align}
where
$$K(\x,X)=[K(\x,\x_1),...,K(\x,\x_N)],$$
$$  K(X,X)=\begin{bmatrix}
    K(\x_1,\x_1) &  \dots  & K(\x_1,\x_N) \\
    \vdots &  \ddots & \vdots \\
    K(\x_N,\x_1)  &  \dots  & K(\x_N,\x_N) 
\end{bmatrix}    ,    $$
{ and $I$ is the identity matrix of size equal to $K(X,X).$}
As the derivative of a Gaussian process is still a Gaussian process with mean given by the derivative of its mean \citep{rasmussen2006gaussian}, by the linearity of the expectation  and the definition of vector projection, we obtain that 
\begin{align}
\label{eqn:MeanDerivativeOrthogonal}
&\langle\nabla_\x^\perp z(\x)\rangle_{p( z(\x)\mid D_N)}=\\
&\quad  \big(\sum_{i=1}^d {v_i} \frac{\partial K(\x,X)  }{\partial x_i } \big)  \big( K(X,X) + \sigma^2I \big)^{-1}[\mathbf{y}_1,...,\mathbf{y}_N]^T \cdot \mathbf{v},
\end{align}
where $\mathbf{v}=(v_1,...,v_d)$ is a unit vector orthogonal to $\mathcal{M}$ in $\x$, which always exists if $\mathrm{dim}(\mathcal{M})<\mathrm{dim}\left(X \right).$
Hence, if we can show that for any $\x' \in \mathcal{M}$, $\sum_{i=1}^d {v_i} \frac{\partial K(\x,\x')  }{\partial x_i } =0$, then the proof is concluded. In order to do that, we notice that, as $\mathcal{M}$ is a subspace, the orthogonal direction $\mathbf{v}$ is a vector of the orthogonal complement of $\mathcal{M}$. Consequently, it holds that
$$ \forall \x \in \mathcal{M}, \quad \sum_{i=1}^d {v}_i x_i =0. $$
From this, for the case $K(\x,\x')=g(\sum_{i=1}^d x_i x_i')$,  we obtain that for any $\x' \in \mathcal{M}$ it holds that
\begin{align*}
&\sum_{i=1}^d {v_i} \frac{\partial K(\x,\x')  }{\partial x_i } =\\
& \sum_{i=1}^d {v}_i g'(\sum_{j=1}^d x_j x_j')x_i' = \\
&g'(\sum_{i=1}^d x_i x_i')\big(  \sum_{i=1}^d {v}_i x_i' \big)  = 0. 
\end{align*}
The case $K(\x,\x')=g(\sum_{i=1}^d(x_i-x_i')^l)$ follows similarly using the fact that subspaces are closed under linear combination.
\end{proof}

\begin{corollary}
\label{Corollary:GaussianLimitCancellation}
Let $f(\x,\w)$ be a BNN trained on a regression problem with data set $D_N=\{ (\x_i,\mathbf{y}_i) \; |\; i=1,\ldots,N \}$ and additive { i.i.d.} zero mean Gaussian observation noise of variance  $\sigma^2>0$.
Assume that:
    \begin{itemize}
        \item To each weight and bias are associated independent normal priors.
        \item $\mathcal{M}$ is a linear subspace of $X\subseteq \mathbb{R}^d$.
        \end{itemize}
Then, for an infinitely-wide BNNs  it holds that
    \begin{equation}
       \langle\nabla_\x^\perp f(\x,\w)\rangle_{p(f(\x,\w)\mid{D_N)}}= 0.
    \label{eq:expected_gradient}
    \end{equation}

\end{corollary}
\begin{proof}
In the limit of infinitely-wide architecture, thanks to Proposition \ref{prop:PriorsWeighttoFunction}, $f(\x,\w)$ converges weakly to a GP. However, in general, the resulting kernel will not directly match the kernels
satisfying the assumption of Theorem \ref{Th:GaussianLimitCancellationGaussian}. 
Rather, in the case of BNNs with fully connected architectures, it will be of the form \citep{matthews2018gaussian}\footnote{Note that the result also holds for convolutional BNNs where for the convolutional kernel instead, each monomial could be weighted differently depending on the number of filters, size and stride of the kernel.}
$$ K(\x,\x')=g(\x \x,\x \x',\x' \x'), $$ where $\x \x'$ is the dot product between vectors $\x$ and $\x'$.
However, by the chain rule for multivariable functions, we get that
\begin{align}
\nonumber
&\frac{\partial K(\x,\x')}{\partial x_i}= \\
& \hspace{1cm} 2x_i D_1 [g(\x \x,\x \x',\x' \x')] +x_i' D_2 [g(\x \x,\x \x',\x' \x')],
\label{Eqn:MultiLayerCase}
\end{align}
where $D_i [g]$ indicates the partial derivative of function $g$ w.r.t.\ argument $i$. Consequently, being both $\x$ and $\x'$ in the data manifold $\mathcal{M}$, the derivations in the proof of Theorem \ref{Th:GaussianLimitCancellationGaussian} apply to each of the argument of Eqn \eqref{Eqn:MultiLayerCase}.
To conclude the proof it is enough to notice that Proposition 1 in  \citep{hron2020exact} guarantees that the BNN posterior converges weakly to the posterior induced by the GP
limit of the prior. Consequently, 
being $$\mathcal{G}_\x:f\mapsto\nabla_\x^\perp f(\x,\w),$$ a linear operator \citep{papoulis2002probability} and bounded by assumption, we have that 
$$ \langle\nabla_\x^\perp f(\x{ ,\w})\rangle_{p(f(\x,\w)\mid{D_N)}}\to \langle\nabla_\x^\perp z(\x)\rangle_{p( z(\x)\mid D_N)}=0, $$
where $\langle\nabla_\x^\perp z(\x)\rangle_{p( z(\x)\mid D_N)}$ is as defined in Theorem \ref{Th:GaussianLimitCancellationGaussian}.
\end{proof}

Note that the assumptions in Theorem \ref{Th:GaussianLimitCancellationGaussian} on $K$ are relatively mild. In fact, not only they include the kernels given by the limiting GP of most common BNN architectures including also convolutional neural networks \citep{novak2018bayesian}, but they are also satisfied by most kernels used in practice for learning with GPs. Consequently, our results also apply to recent approaches that encode informative functional priors for BNNs as Gaussian processes \citep{tran2022all}. Note also that an implicit assumption of Theorem \ref{Th:GaussianLimitCancellationGaussian} and Corollary \ref{Corollary:GaussianLimitCancellation} is that the partial derivatives at $\x$ are well defined. In the case of ReLU activation functions, this may not be always the case. In these cases, the results can be equivalently stated by considering subderivatives. 

We should also stress that Theorem \ref{Th:GaussianLimitCancellationGaussian} and Corollary \ref{Corollary:GaussianLimitCancellation} hold for any size of $D_N$. Of course, in this general setting, the projection of the expected gradient of the loss to a direction parallel to $\mathcal{M}$ may not be zero, but will only point to directions where the loss increases within the data manifold. Consequently, if a BNN has small loss everywhere in the data manifold, then necessarily $\langle\nabla_\x^\perp f(\x)\rangle_{p(f(\x,\w)\mid{D_N)}}$ will vanish for any $\x \in \mathcal{M}$.

\begin{remark}
\label{Remak:Variance}
In Theorem \ref{Th:GaussianLimitCancellationGaussian} we only consider the expectation of the orthogonal gradient. In fact, we remark that Theorem \ref{Th:GaussianLimitCancellationGaussian} does not imply that also the variance of the orthogonal projection of the GP vanishes. In particular, in general, the variance of the orthogonal gradient is non-zero, showing how the robustness properties occur only in expectation.
{  This will be empirically investigated in Section \ref{sec:ExperimentalGPtoBNN}.}

\end{remark}

\begin{remark}
The assumption that $\mathcal{M}$ is a linear subspace is needed to guarantee that all training points are orthogonal to the same vector defining the orthogonal direction in $\x$. This guarantees perfect cancellation. However, in the more general case, where $\mathcal{M}$ is not a linear subspace, it is reasonable to expect that Theorem \ref{Th:GaussianLimitCancellationGaussian} and Corollary \ref{Corollary:GaussianLimitCancellation} can still often approximately hold. In fact, for many kernels only training points close to the test point $\x$ will influence the prediction at $\x$. Consequently, for the cancellation in Theorem \ref{Th:GaussianLimitCancellationGaussian} to hold it would be enough that the orthogonal direction in $\x$ remains approximately orthogonal only in a neighbourhood of $\x;$ thus extending our result to more general $\mathcal{M}.$ 
\end{remark}

\subsection{Extension to Classification Setting}
\label{Sec:Classification}

Theorem \ref{Th:GaussianLimitCancellationGaussian} is stated for regression, which is a particularly favourable setting for analysis because of the existence of closed form solutions for the GP posterior. Fortunately, in Theorem \ref{th:MainTheoremClassification} and Corollary \ref{corollary:ClassificationCase} below, we show that our results also extend to classification.

\begin{theorem}
\label{th:MainTheoremClassification}
    Let $z(\x)$ be a zero-mean Gaussian process (GP) with covariance function $K:\mathbb{R}^d \times \mathbb{R}^d \to \mathbb{R}.$ Call $z(\x)\mid D_N$ the posterior GP obtained by training $z$ on a classification problem with data set $D_N=\{ (\x_i,\mathbf{y}_i) \; |\; i=1,\ldots,N, y_{i}\in\{0,1\} \}$. 
Then, under the same assumptions of Theorem \ref{Th:GaussianLimitCancellationGaussian}, for any $\x\in \mathcal{M}$ it holds that
    \begin{equation}
       \langle\nabla_\x^\perp z(\x)\rangle_{p( z(\x)\mid D_N)} = 0.
    \label{eq:expected_gradient_GP}
    \end{equation}
\end{theorem}
\begin{proof}
Because of the non-Gaussianity of the likelihood,
$z(\x)\mid D_N$ is not Gaussian in general. However, it is still possible to show that (see e.g., Eqn 3.22 in \citep{rasmussen2006gaussian})
\begin{align}
\label{Eqn:PosteriorCancellationClassification}
    \langle z(\x)\rangle_{p( z(\x)\mid D_N)} = K(\x,X) K(X,X)^{-1} g(D_N),
\end{align} 
where $g$ is a function given by the expectation of the latent function given the training data.  Consequently, as $g$ is independent of $\x$, under boundeness assumption of $K(X,X)^{-1} g(D_N)$, the same approach considered in the proof of Theorem \ref{Th:GaussianLimitCancellationGaussian} following Eqn \eqref{Eqn:PosteriorGPRegression}
holds also in this setting.
\end{proof}

\begin{corollary}
\label{corollary:ClassificationCase}
Let $f(\x,\w)$ be a BNN trained on a classification problem with data set $D_N=\{ (\x_i,\mathbf{y}_i) \; |\; i=1,\ldots,N, y_{i}\in\{0,1\} \}$. 
Then, under the same assumptions of Theorem \ref{Th:GaussianLimitCancellationGaussian}, for any $\x\in \mathcal{M}$ it holds that
$$ \langle\nabla_\x^\perp f(\x,\w)\rangle_{p(f(\x,\w)\mid{D_N)}}\to 0. $$
\end{corollary}
\begin{proof}
The proof is analogous to that of Corollary \ref{Corollary:GaussianLimitCancellation} by noticing that in the infinitely-width limit (trained) BNNs converge to GPs also in the classification setting \citep{hron2020exact}. 
\end{proof}

}

\section{Consequences and Limitations of our Results}
\label{sec:limitations}

The results in the previous section have the natural consequence of protecting BNNs against gradient-based attacks, due to the vanishing average of the expectation of the gradients in the limit. Its proof also sheds light on a number of observations made in recent years. Before moving on to empirically validating the theorem in the finite-width case, it is worth reflecting on some of its implications and limitations:
\begin{itemize}
    \itemsep0em 
    \item Theorem \ref{Th:GaussianLimitCancellationGaussian}  holds in a specific thermodynamic limit, however it is reasonable to expect that the averaging effect of BNN gradients can still often provide considerable protection in conditions when the network architecture leads to high accuracy and strong expressivity.  
    \item Our results holds when the { prediction is obtained by averaging w.r.t.\ the true posterior. Unfortunately, for neural networks it is generally unfeasible to obtain the true posterior, and cheaper variational approximations are commonly employed \citep{blundell2015weight}. In practice, depending on the quality of the approximation, these can still provide protection from adversarial attacks. }
    \item {   Theorem \ref{Th:GaussianLimitCancellationGaussian} and \ref{th:MainTheoremClassification} are stated for GPs. Consequently, these theorems not only provides theoretical backing to recent empirical observations of the adversarial robustness of GPs \citep{cardelli2019robustness,grosse2021killing,patane2022adversarial}, but also generalizes to all processes that  converge to GPs thanks to the  Central Limit Theorem.}
    \item While the Bayesian posterior ensemble may not be the only randomization to provide protection, it is clear that some simpler randomizations such as bootstrap will be ineffective, as noted empirically by \cite{bekasov2018bayesian}. This is because bootstrap resampling introduces variability along the data manifold, rather than in orthogonal directions. 
    In this sense, bootstrap clearly cannot be considered a Bayesian approximation, especially when the data distribution has zero measure support w.r.t.\ the ambient space.
    \item Our results only guarantees protection against gradient-based adversarial attacks. As a consequence, it is not clear if the robustness properties of BNNs also extend to non-gradient based attacks. Empirical results in Section \ref{sec:non_gradient_attacks} also suggest that the vanishing gradient properties of BNNs may provide an increased robustness against specific gradient-free attacks.
\end{itemize}

\section{Empirical Results}
\label{sec:empirical_results}

\label{sec:ExpRobustnessAccuracyTradeoff}
\begin{figure*}[ht]
    \centering
    \includegraphics[width=1.\linewidth]{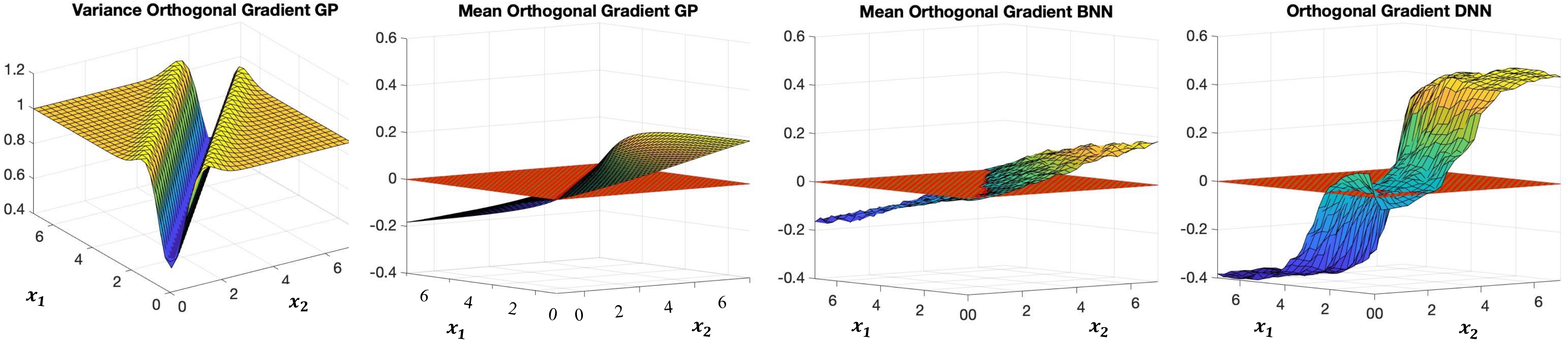}
    \caption{ We consider a regression problem with data manifold given by the line $x_1=x_2$ and data generated by the function $\frac{2x_1^2}{10}-x_1$. {  For this problem we train a BNN with HMC and a deterministic NN (DNN) with SGD of same architecture: relu activation functions and one hidden layer with 512 neurons. Furthermore, we also train a GP with kernel equal to that of an infinitely-wide BNN with relu activation functions and 1 hidden layer.} All learning models achieve accuracy $>99\%$. We plot the mean and variance of the scalar projection of the gradient { of the GP} in a direction orthogonal to the data manifold { for all points in the ambient space}  and compare it to the mean of the same quantity for the BNN and DNN.  Plane $z=0$ is plotted in red. 
    \label{fig:provaBNNCOnvergenza}
    }
\end{figure*}

In this section we empirically investigate our theoretical findings on different BNNs.
We train a variety of BNNs on the MNIST and Fashion MNIST \citep{xiao2017fashion} datasets {(using their standard 60k/10k train-test split)}, and evaluate their posterior distributions using HMC and VI approximate inference methods.
{  Details on the architectural and training hyperparameters used throughout the current section are listed in table form in Appendix Section \ref{sec:hyper_params}. 
In Section \ref{sec:ExperimentalGPtoBNN}, we empirically validate the theoretical results of Section \ref{sec:adv_robustness_bayesian_averaging} on a synthetic regression dataset on which we have access to the data manifold.}
In Section \ref{sec:ExpGradient}, we focus on image classification tasks and experimentally verify the validity of the zero-averaging property of gradients implied by the  {theoretical results discussed in Section \ref{sec:adv_robustness_bayesian_averaging}}, 
and discuss its implications on the behaviours of FGSM and PGD attacks on BNNs in Section \ref{Sec:ExpAttacks}.
In Section \ref{sec:ExpRobustnessAccuracyTradeoff} we analyse the relationship between robustness and accuracy on thousands of different NN architectures, comparing the results obtained by Bayesian and by deterministic training.  Further, in Section \ref{sec:non_gradient_attacks} we investigate the robustness of BNNs on a  gradient-free adversarial attack \citep{athalye2018obfuscated}.  { Experiments were performed with anNVIDIA 2080Ti GPU on a machine with 20-core Intel Core
Xeon 6230 and 256GB of RAM.}
{ 

\subsection{Analysis of the Convergence of BNN gradients}
\label{sec:ExperimentalGPtoBNN}
To investigate the limit behaviour of the orthogonal and parallel gradients of BNNs, in Figure \ref{fig:provaBNNCOnvergenza} we consider a synthetic regression example, where the ambient space is $\mathbb{R}^2$ and the data manifold is the line $x_1=x_2$. We consider data generated by the function $\frac{2x_1^2}{10}-x_1$ and on these data we train both a 1 hidden layer BNN with 512 hidden neurons with {  Hamiltonian Monte Carlo (HMC)} and the corresponding limiting GP, as well as a deterministic NN trained with stochastic gradient descent (SGD) on the same architecture.  For each point in the ambient space we compute the variance and mean of the scalar projection of the expected posterior gradient of the prediction of the model into a unit direction orthogonal to the data manifold. 
As expected (see Remark \ref{Remak:Variance}), we observe that the variance is non-vanishing on the ambient space, that is, both inside and outside the data manifold.
Furthermore, while for GPs and BNNs the mean of the orthogonal projection is respectively zero and approximately zero in the data manifold, for the DNN this is not always the case. Additionally, the absolute value of the expected orthogonal projection in any point in the ambient space is substantially greater for the DNN compared to the GP and BNN case.

\begin{table*}[b]
    \centering
    \begin{tabular}{ l  c  c  r  c  c  c }
        \hline
        Method & Hidden Neurons  & avg ort grad & max ort grad & avg par grad & max par grad  \\ 
        \midrule 
        HMC ($\sigma=0.4$) & 16 & 0.0556 & 0.1116 & 1.1961 & 2.2275  \\
        HMC ($\sigma=0.4$) & 64 &   0.0613 & 0.0836 & 1.2028 & 2.1420  \\
        HMC ($\sigma=0.4$) & 512 & 0.0132 & 0.0355 & 1.1991 & 2.1724  \\
        HMC ($\sigma=0.25$) & 16 &   0.0976 & 0.1508 & 1.2537 & 2.3727 \\
               HMC ($\sigma=0.25$) & 64 &   0.0702 & 0.1052& 1.2402 & 2.3064 \\
        HMC ($\sigma=0.25$) & 512 &  0.0359 & 0.0554 & 1.2487 & 2.3353 \\
        HMC ($\sigma=0.1$) & 16 & 0.1127 & 0.1921 & 1.3214 & 2.6608  \\
          HMC ($\sigma=0.1$) & 64 & 0.0539 & 0.1271 & 1.3171 & 2.6290  \\
        HMC ($\sigma=0.1$) & 512 & 0.0479  & 0.0900 & 1.3251 & 2.6669 \\
        SGD & 16 & 0.4337 & 0.8056 & 1.5525& 3.1999 \\
        SGD & 64 & 0.2612 & 0.6676 & 1.5435 & 3.2238  \\
        SGD & 512 & 0.1234 & 0.2491 & 1.5262 & 3.1424 \\
        \hline
    \end{tabular}
    \caption{For the setting described in Figure \ref{fig:provaBNNCOnvergenza} we trained various neural networks with HMC and SGD. Average orthogonal and parallel gradients and maximum orthogonal and parallel gradients are computed over $70$ points sampled from the data manifold. All architectures have 1 hidden layer. }
    \label{tab:setup}
\end{table*}

In order to further investigate the qualitative behaviour observed in Figure \ref{fig:provaBNNCOnvergenza}, we consider the same setting of Figure \ref{fig:provaBNNCOnvergenza} and, in Table \ref{tab:setup}, we report the average orthogonal and parallel scalar projection of the gradient for various BNN architectures  and compare it against the same architectures trained with SGD. In particular, we consider BNNs trained with HMC with different values of $\sigma $,  the standard deviation of the normal likelihood employed in regression training, and for different number of neurons. All architectures achieve expected loss $\leq 0.02$ and average and maximum gradients are evaluated on the same $70$ points sampled from the data manifold.  As expected from Corollary \ref{Corollary:GaussianLimitCancellation}, for all values of $\sigma$ both average and maximum orthogonal gradients decrease with a { wider} network, while the parallel gradient remains approximately constant. For SGD the orthogonal gradients are in all cases substantially larger compared to those of BNNs. However, interestingly, even in the SGD case, { wider} networks achieve lower orthogonal gradients. This suggests that the limit established in Corollary \ref{Corollary:GaussianLimitCancellation} and \ref{corollary:ClassificationCase} may also partly benefit neural networks trained with SGD, although the convergence is much slower.
Related to this, we also note that $\sigma$ has a substantial impact on the robustness of the BNN. This can be understood because, as shown in Figure \ref{fig:LEarnedBNNManifold} in the Appendix, the lower $\sigma$ the smaller the variance of the posterior BNN and the larger the orthogonal gradient of the limit GP (see Figure \ref{fig:ComparisionOrthogonalDerivativGP}). This indicates how the uncertainty can be beneficial in increasing the robustness of the predictions and how training and model parameters can play an essential role for BNN robustness.

\begin{figure}[ht]
    \centering
    \includegraphics[width=1.\linewidth]{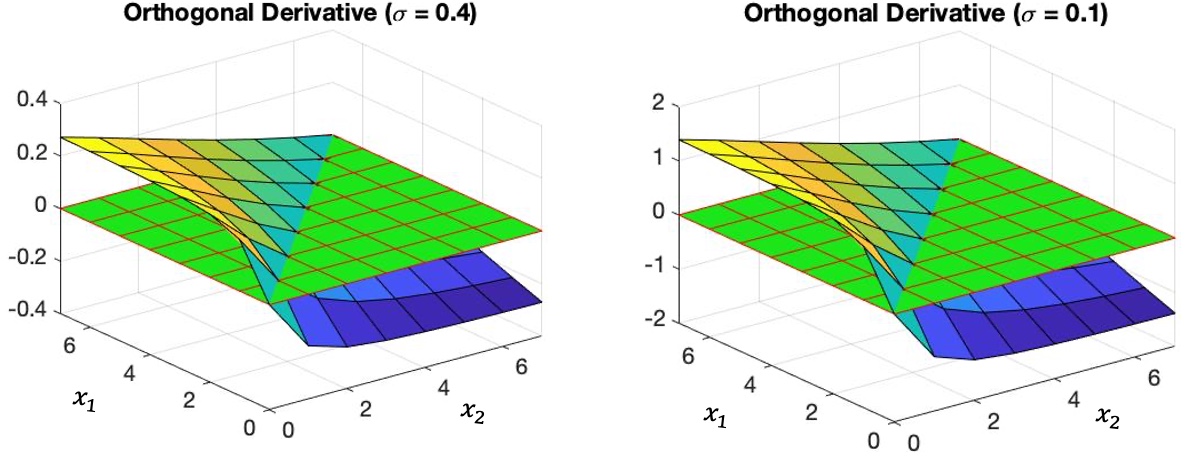}
    \caption{We plot the scalar projection of the orthogonal gradient of the GP limit of neural networks with ReLU activation functions and one hidden layer for $\sigma=0.1$ and $\sigma=0.4$ for the settings of Figure \ref{fig:provaBNNCOnvergenza}, where $\sigma$ is the standard deviation of { the} likelihood. It is possible to observe how in both cases the orthogonal derivative is identically $0$ on the data manifold. However, outside of the data manifold $\sigma$ has a large effect on the orthogonal derivative.
    }
    \label{fig:ComparisionOrthogonalDerivativGP}
\end{figure}

}

\subsection{Evaluation of the Gradient of the Loss for BNNs on Image Classification Tasks}
\label{sec:ExpGradient}

We investigate the vanishing behavior of input gradients of the loss - established for the BNN classification case by Corollary~\ref{corollary:ClassificationCase} for the limit regime - in the finite, practical settings, that is with a finite-width BNN. 
Specifically, { we consider various neural network architectures (those that will be reported in Section \ref{sec:ComparisionVIHMC}) trained with both HMC and VI with Bayes by Backprop \citep{blundell2015weight} and} we select the architectures achieving the highest test accuracy: a two hidden layers BNN (with 1024 neurons per layer) for HMC and a three hidden layers BNN (512 neurons per layer) for VI. 
These achieve approximately $95\%$ test accuracy on MNIST and $89\%$ on Fashion MNIST when trained with HMC; as well as $95\%$ and $92\%$, respectively, when trained with VI.  Details about the hyperparameters used for training can be found in Tables \ref{tbl:bnns_hmc_1} and \ref{table:bnns_vi} listed in Appendix Section \ref{sec:hyper_params}.

\begin{figure}[ht]
\centering
\includegraphics[width=.9\columnwidth]{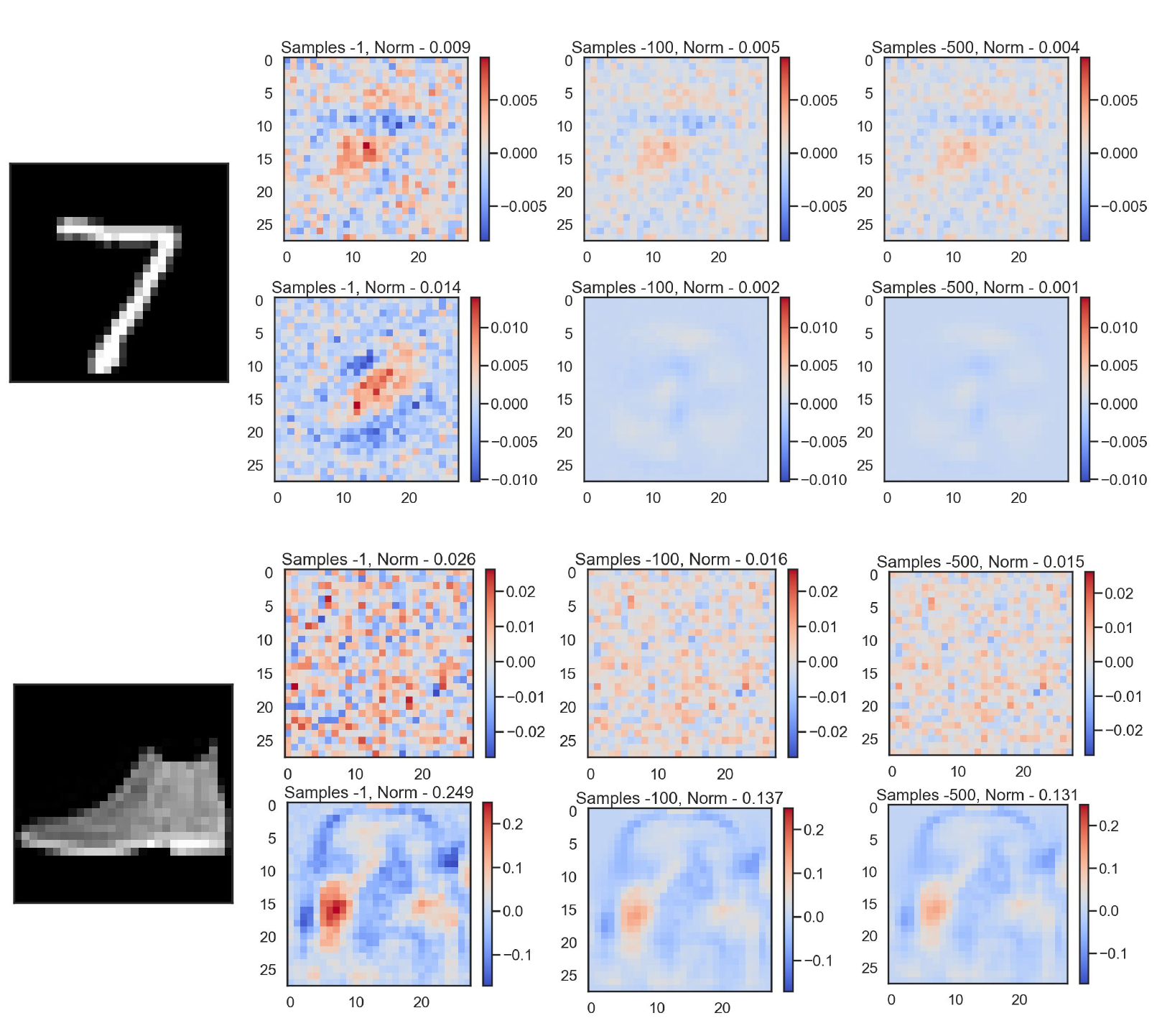}
\caption{
We plot the input gradient (and report its $\ell_\infty$-norm under ``Norm'' in the title of each plot) of the expected loss gradients for  two BNNs trained on MNIST  (top rows) and Fashion MNIST  (bottom rows) for some example images and for different number of samples from the posterior predictive distribution. For training the BNN on MNIST we employ HMC (top most row of next to each image), and VI (bottom most row next to each image). To the right of the images, we plot a heat map of gradient values. In all cases we observe that the expected loss gradients decrease when increasing the number of samples.
}
\label{fig:gradient-heatmaps}
\end{figure}

Figure~\ref{fig:gradient-heatmaps} investigates the behaviour of the component-wise expectation of the loss gradient as more samples from the posterior distribution are incorporated into the BNN predictive distribution.
We observe that as the number of samples taken from the posterior distribution of $\mathbf{w}$ increases, all the components of the gradient decrease in absolute value.
Notice that the gradients of the individual NNs (that is those with just one sampled weight vector), are markedly larger. 
This is also confirmed in Figure~\ref{fig:expLossGradients_stripplot_HMC}, where we provide a systematic analysis of the aggregated gradient convergence properties on $250$ test images for MNIST and Fashion MNIST.  {We plot the decrease in the maximum gradient magnitude (with shaded error regions corresponding to the square root of the variance) as we increase the number of samples used to approximate the input gradient.}
We observe that for both HMC and VI the magnitude of the gradient components drops as the number of samples increases. 
Notice that the gradients computed on HMC trained networks drops more quickly and towards a smaller value compared to VI trained networks. 
This is in accordance to what is discussed in Section \ref{sec:limitations}, as VI introduces additional approximations in the Bayesian posterior computation.
%
\begin{figure}[ht]
\centering
\includegraphics[width=1.\columnwidth]{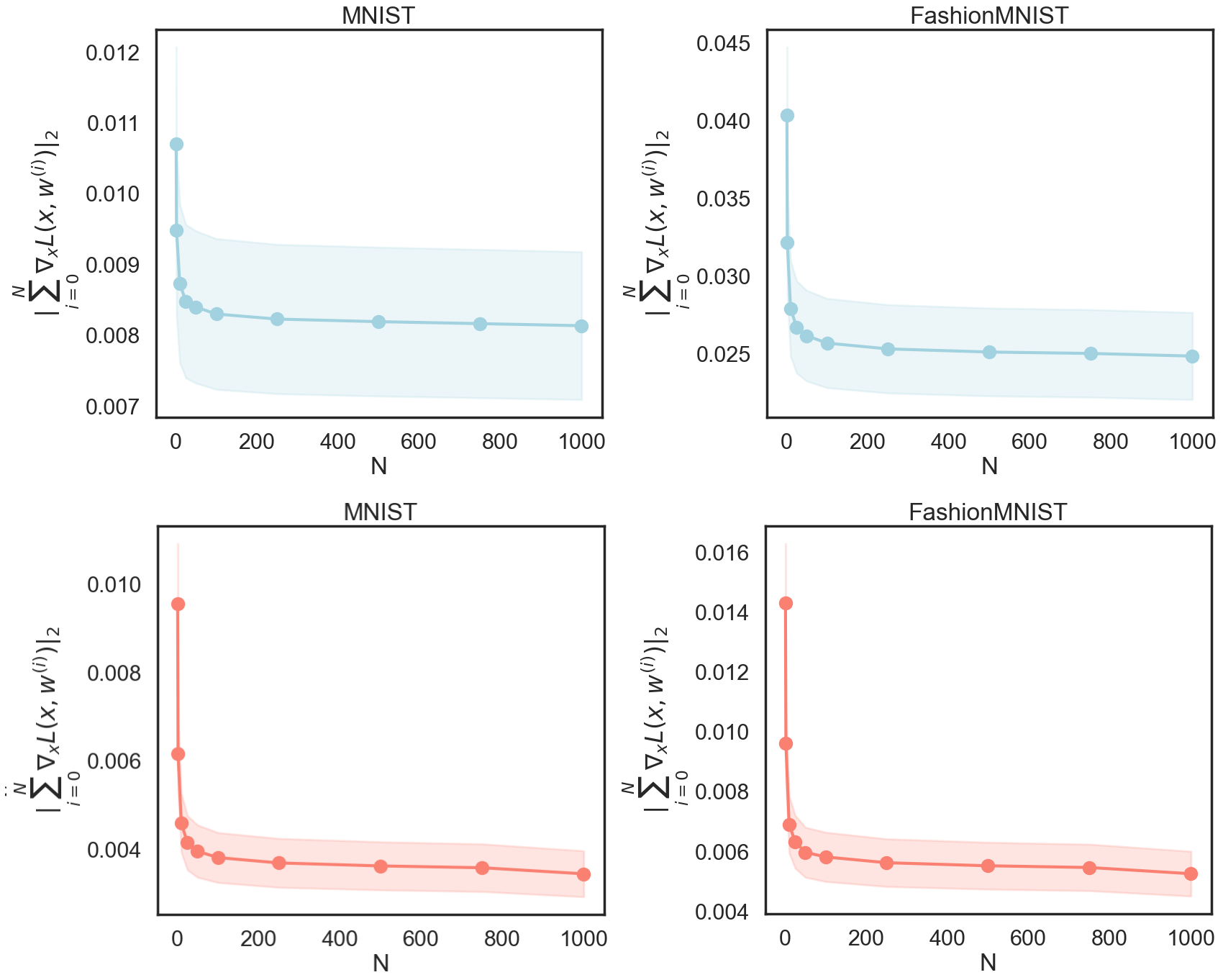}
\caption{
We plot the $\ell_2$ norm of the input gradient as we increase the number of samples. In the top row (in blue) we plot the trend of the input gradient for a BNN { trained with VI} on MNIST (left) and FashionMNIST (right). In the bottom row (in red) we plot the trend of the input gradient for an HMC-trained BNN on MNIST (left) and FashionMNIST (right). In all cases, we observe the expected trend that the norm of the gradient decreases as we increase the number of samples. }
\label{fig:expLossGradients_stripplot_HMC}
\end{figure}

\subsection{Gradient-Based Attacks for BNNs}
\label{Sec:ExpAttacks}
The fact that gradient cancellation occurs in the limit does not directly imply that BNNs' predictions are robust to gradient-based attacks in the finite case.
For example, FGSM attacks \citep{goodfellow2014explaining} are crafted such that the direction of the manipulation is given only by the sign of the expectation of the loss gradient and not by its magnitude. 
Thus, even if the entries of the expectation drop to an infinitesimal magnitude but maintain a meaningful sign, then FGSM could potentially produce effective attacks.
In order to test the implications of vanishing gradients on the robustness of the posterior predictive distribution against gradient-based attacks, we compare the behaviour of FGSM, PGD\footnote{With 15 iterations and 1 restart.} \citep{pgdattack} {and CW \cite{carlini2016evaluating} attacks, arguably the most commonly employed gradient-based attacks}. 
\begin{table}\begin{center}
  \caption{Adversarial robustness { defined as the proportion of test samples for which the attack is not successful for}  BNNs trained with HMC and VI with respect to  FGSM, PGD  {and CW}.
  }
   {\begin{tabular}{c|c|c|c|c}
    \toprule
    \textbf{Dataset/Method} & \textbf{Acc.} & \textbf{FGSM} & \textbf{PGD} & \textbf{CW} \\
    \midrule
     MNIST/SGD & 0.984 & 0.376 & 0.334 & 0.358 \\ \hline
     MNIST/VI & 0.974 & 0.623 & 0.508 & 0.512 \\ \hline
    MNIST/HMC & 0.914 & 0.876 & 0.864 & 0.859 \\ \hline
    Fashion/SGD & 0.891 & 0.510 & 0.410 & 0.432   \\ \hline
    Fashion/VI & 0.862 & 0.622 & 0.581 & 0.588   \\ \hline
    Fashion/HMC & 0.828 & 0.661 & 0.639 & 0.650  \\ \bottomrule
  \end{tabular}}
\label{tab:attackcomp}
\end{center}\end{table}

In particular, in  {Table~\ref{tab:attackcomp} we evaluate a single hidden-layer neural network architecture with 512 hidden neurons trained on MNIST and Fashion MNIST using HMC, { , VI (Bayes by Backprop \citep{blundell2015weight}),} , and SGD. We then attack it using attack strength $\epsilon=0.1$ for MNIST and $\epsilon=0.05$ for FashionMNIST.
For each image, we compute the expected gradient using 50 posterior samples. 
Details on the hyper-parameters used in each setting can be found in Tables \ref{table:MNIST_hmc}--\ref{table:MNIST_SGD} listed in Appendix Section \ref{sec:hyper_params} (entries marked with a *).
We observe that BNNs are genereally more robust against each of the gradient-based attacks. 
Interestingly, we also notice that for BNNs, FGSM performs comparably to PGD and CW attacks with  little increase in attack success rate when using stronger attacks. Furthermore, we should also note that for FashionMNIST, the attacks are more successful than in the MNIST case. 
As FashionMNIST poses a more complicated classification problem than MNIST, the BNNs obtain less accuracy in the former. In agreement with the discussion in Section \ref{sec:limitations}, this implies higher loss and that the conditions set up in the main theorems and corollaries are less approximately met.
} 

\subsection{Robustness Accuracy Analysis in Deterministic and Bayesian Neural Networks}
\label{sec:ComparisionVIHMC}
In Section \ref{sec:limitations}, we noticed that,  {as proxy of the loss and of convergence}, high accuracy might be related to high robustness to gradient-based attacks for BNNs.
Notice, that this would run counter to what has been observed for deterministic NNs trained with SGD \citep{su2018robustness}.

In this section, we look at an array of more than 1000 BNNs with different hyperparameters trained with HMC and VI { (Bayes by Backprop \citep{blundell2015weight})}  on MNIST and Fashion MNIST (details on training and architecture hyper-parameters explored can be found in Tables~\ref{table:MNIST_hmc}--\ref{table:MNIST_SGD} listed in Appendix Section \ref{sec:hyper_params}).
We experimentally evaluate their accuracy/robustness trade-off on FGSM attacks as compared to that of the same neural network architectures trained via standard (i.e., non-Bayesian) SGD based methods.
For the robustness evaluation we consider the average difference in the softmax prediction between the original test points and the crafted adversarial input, as this provides a quantitative and smooth measure of adversarial robustness that is closely related with misclassification ratios \citep{cardelli2019statistical}.
That is, for a collection of $N$ test point, we compute 
\begin{align}
\label{Eqn:softmaxDifference}
\frac{1}{N}    \sum_{j=1}^{N}   \big| \langle f(\mathbf{x}_j,\mathbf{w}) \rangle_{p(\mathbf{w}|D)} -  \langle f(\tilde{\mathbf{x}}_j,\mathbf{w}) \rangle_{p(\mathbf{w}|D)} \big|_\infty.
\end{align}

\label{sec:ExpRobustnessAccuracyTradeoff}
\begin{figure*}[ht]
    \centering
    \includegraphics[width=.9\linewidth]{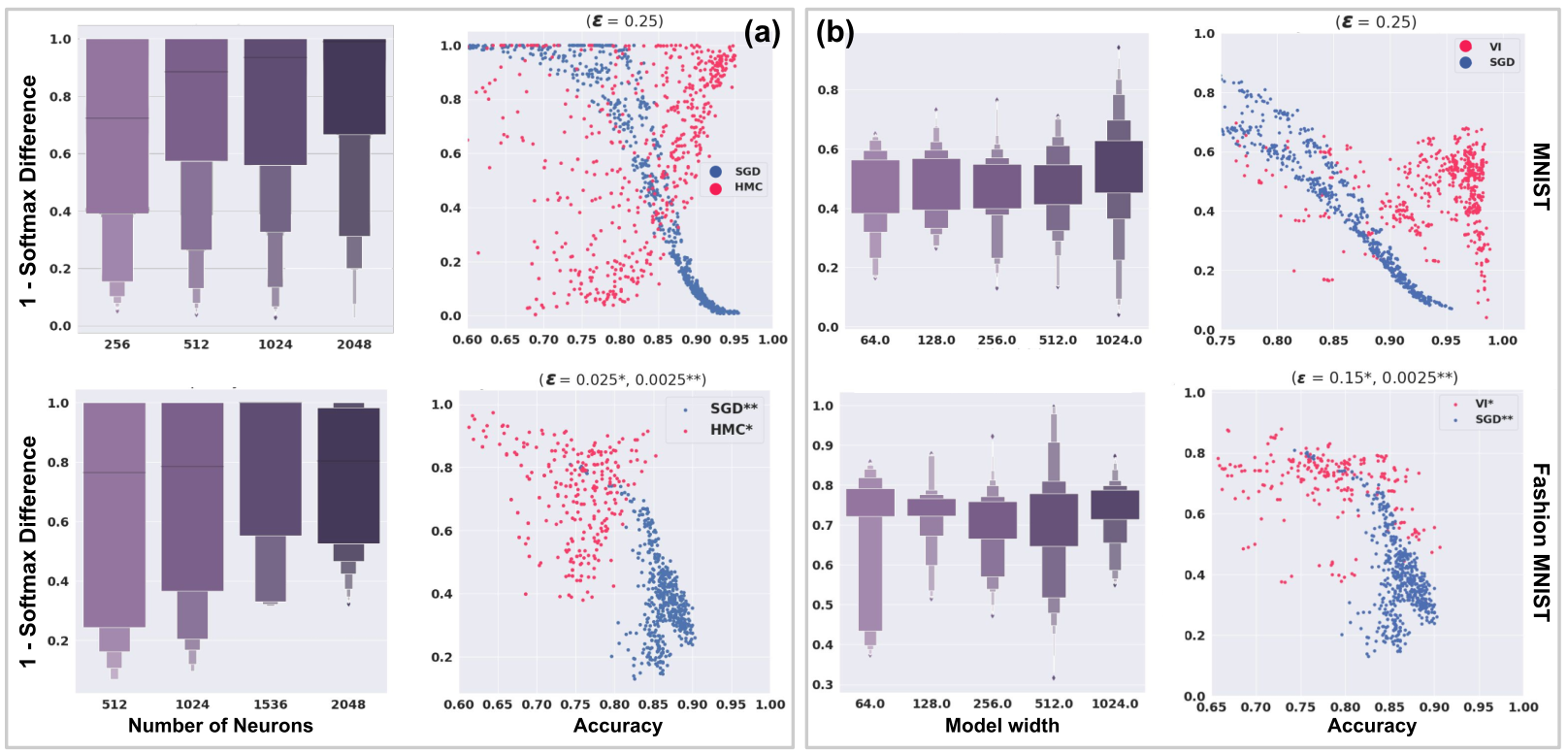}
    \caption{Robustness-Accuracy trade-off on MNIST (first row) and Fashion MNIST (second row) for BNNs trained with HMC (a), VI (b) and SGD (blue dots){ , where softmax difference is computed according to Eqn \eqref{Eqn:softmaxDifference} and denotes denotes the average maximal difference in softmax value for the specific neural network for an input $\epsilon-$ball of input points computed via FGSM attack. } While a trade-off between accuracy and robustness occur for deterministic NNs, experiments on HMC show a positive correlation between accuracy and robustness. The boxplots show the correlation between model capacity and robustness.
    Different attack strength ($\epsilon$) are used for the three methods accordingly to their average robustness.
    }
    \label{fig:bayesmnist}
\end{figure*}

The results of the analysis are plotted in Figure \ref{fig:bayesmnist}.
Each dot in the scatter plots represents the results obtained for each specific network architecture trained with SGD (blue dots), HMC (pink dots in plots (a)) and VI (pink dots in plots (b)).
As already reported in the literature \citep{su2018robustness}
we observe a marked trade-off between accuracy and robustness (i.e., 1 - softmax difference) for high-performing deterministic networks.
Interestingly, this trend is  reversed for BNNs trained with HMC (plots (a) in Figure \ref{fig:bayesmnist}) where we find that as networks become more accurate, they additionally become more robust to FGSM attacks as well. 
We further examine this trend in the boxplots that represent the effect that the network width has on the robustness of the resulting posterior.
We find the existence of an increasing trend in robustness as the number of neurons in the network is increased.
This is in line with  our theoretical findings, i.e., as the BNN approaches the infinite width limit, the conditions for Corollary \ref{corollary:ClassificationCase} are approximately met and the network is protected against gradient-based attacks.

On the other hand, the trade-off behaviours are less obvious for the BNNs trained with VI and on Fashion-MNIST.
In particular, in plot (b) of Figure \ref{fig:bayesmnist} we find that, similarly to the deterministic case, also for BNNs trained with VI, robustness seems to have a negative correlation with accuracy, albeit less marked than for SGD.
Furthermore, { similarly than for HMC, we also observe that for VI the robustness of the BNNs tend to increase with the width of the network. However, the trend is less clear compared to the HMC case. This can be linked to the fact that VI is known to often under-approximate uncertainty \citep{rasmussen2006gaussian}, which may lead to posterior with excessively small variance, thus behaving similarly to a deterministic NN.  }
{ We should also emphasize that experimental results in this paper have been performed with Bayes by backprop \cite{blundell2015weight}. However, approximate Bayesian inference techniques for deep learning is an active area of research, including recent developments, e.g., SWAG \citep{maddox2019simple} or Noisy Adam \citep{zhang2018noisy}, such developments could, in principle, lead to a behaviour closer to that exhibited by HMC.}

\subsection{Gradient-Free Adversarial Attacks}
\label{sec:non_gradient_attacks}

In this section, we empirically evaluate the most accurate BNN posteriors on MNIST and Fashion MNIST from Figure \ref{fig:bayesmnist} against gradient-free adversarial attacks. Specifically, we consider ZOO \citep{chen2017zoo}, a gradient-free adversarial attack based on a finite-difference approximation of the gradient of the loss obtained by querying the neural network (in the BNN case the attacker queries the posterior distribution). In particular, we selected ZOO because it has been shown to be effective even when tested on networks that purposefully obfuscate their gradients or have vanishing gradients \citep{athalye2018obfuscated}. 
In Figure~\ref{fig:non-gradient-attacks} we observe that similarly to gradient-based methods, ZOO is substantially less effective on BNNs compared to deterministic NNs in all cases, with BNNs again achieving both high accuracy and high robustness simultaneously. 
Furthermore, once  again HMC is more robust to the attack than VI, which is in turn substantially more robust than deterministic NNs. This suggests how, similarly to what observed in the previous subsections, a more accurate posterior distribution may lead to a more robust model also to gradient-free adversarial attacks. 


\begin{figure*}[ht]
\centering
\includegraphics[width=0.8\textwidth]{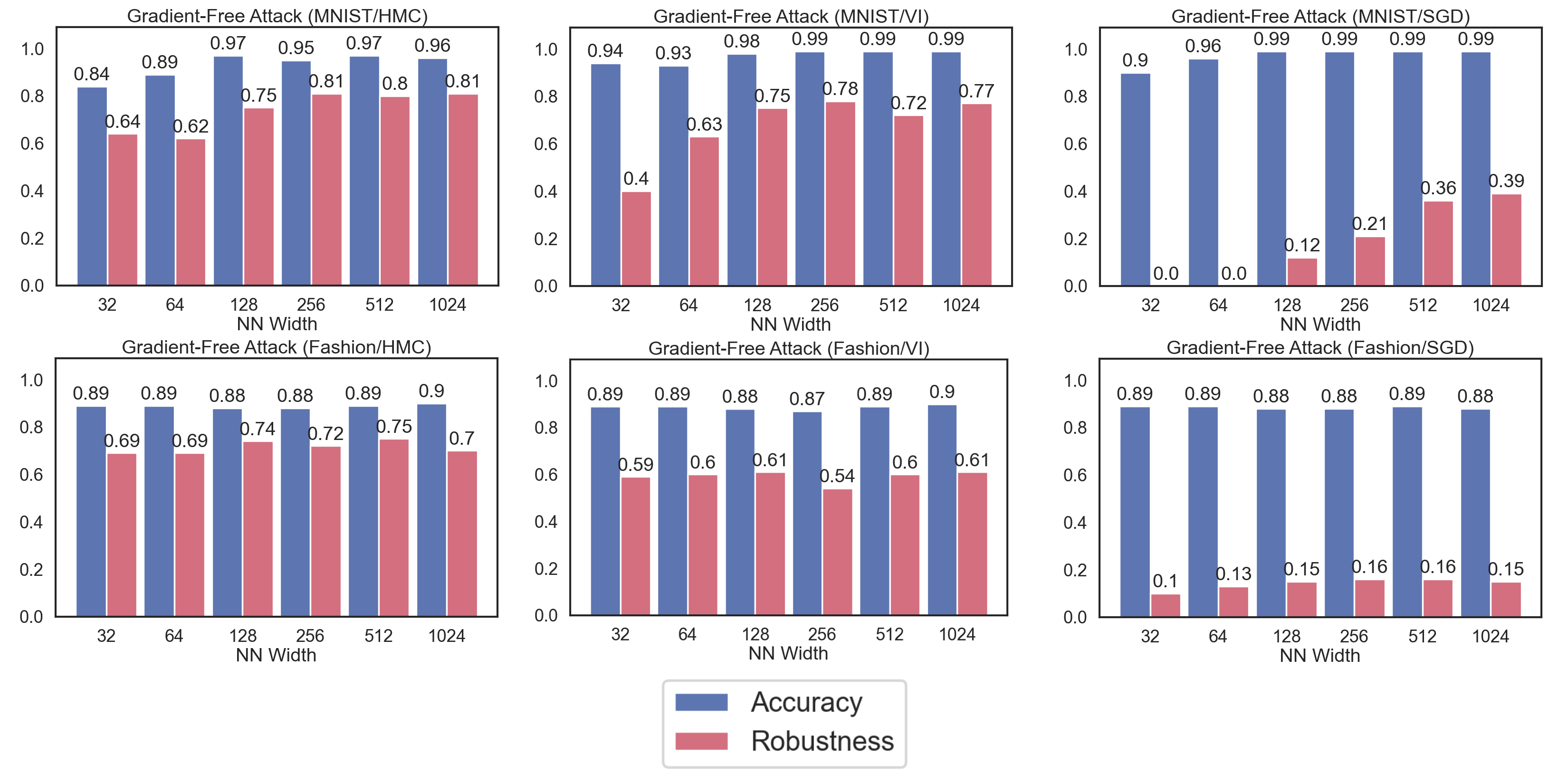}
\caption{
Gradient-free adversarial attacks on BNNs display similar behavior to gradient-based attacks. We evaluate the more accurate networks from Figure~\ref{fig:bayesmnist} with gradient-free attacks on MNIST (first row) and Fashion MNIST (second row) for BNNs trained with HMC (left column), VI (center column), and SGD (right column). We use the same attack parameters as in the Figure~\ref{fig:bayesmnist}, but use ZOO as an attack method. 
}
\label{fig:non-gradient-attacks}
\end{figure*}

\section{Conclusions}
The quest for robust, data-driven models is  an essential component towards the construction of AI-based technologies. 
In this respect, we believe that the fact that Bayesian ensembles of NNs can provide additional robustness against a broad class of adversarial attacks will be of great relevance. 
While promising, this result comes with some significant limitations. First, {  while our empirical results present encouraging examples of robustness,} our theoretical analysis is  { only guaranteed to hold for infinite neural networks}.  Secondly, and perhaps more importantly, performing Bayesian inference in large non-linear models is extremely challenging. In fact, while in our experiments cheaper approximations, such as VI, also enjoyed a degree of adversarial robustness, { albeit reduced compared to NNs trained with more accurate Bayesian inference methods}, there are no guarantees that this will hold in general. To this end, we hope that this result will spark renewed interest in the pursuit of efficient Bayesian inference algorithms, {which have the potential to lead to learning models that are intrinsically accurate and robust.}
%

\bibliography{bibliography}
\clearpage
\appendix

\subsection{Hyperparameters used for Training}\label{sec:hyper_params}

\begin{figure*}[ht]
    \centering
    \includegraphics[width=.75\linewidth]{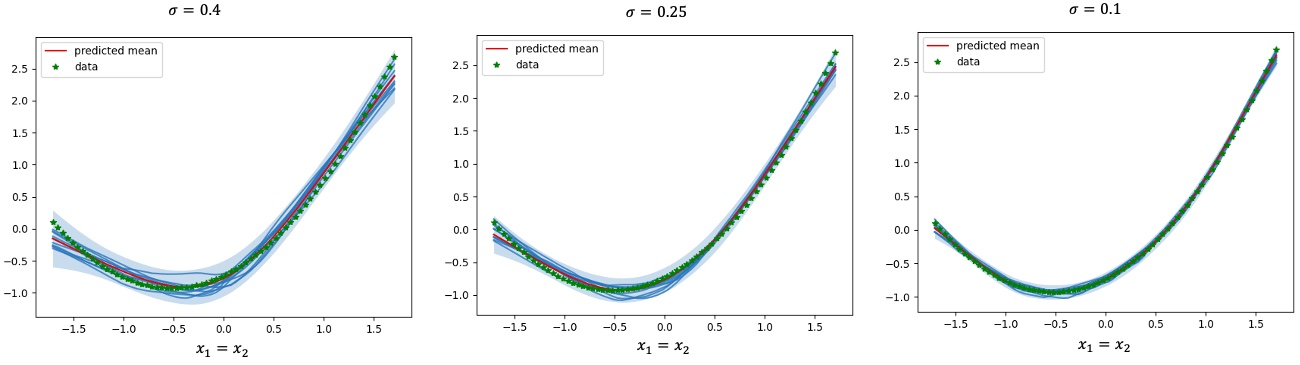}
    \caption{For the setting in Figure \ref{fig:provaBNNCOnvergenza} we plot the prediction of $3$ different BNNs trained with HMC. All architectures have 1 hidden layers with $64$ hidden neurons and in all cases the same training parameters have been used except for the  $\sigma,$ the standard deviation of the observation noise. 
    }
    \label{fig:LEarnedBNNManifold}
\end{figure*}

Details on hyperparameters used for training of BNNs and deterministic NNs  are listed in Tables \ref{tbl:bnns_hmc_1}--\ref{table:MNIST_SGD}.

\begin{table}[ht] {
\centering
\caption{
Hyperparameters for the BNN using HMC of Figures \ref{fig:gradient-heatmaps} and \ref{fig:expLossGradients_stripplot_HMC}.
}
\begin{tabular}{ p{2.5cm}||p{2cm}|p{2cm}}
 \multicolumn{3}{c}{\textbf{Training hyperparameters for HMC}}\\
 \toprule
   Dataset & MNIST & Fashion MNIST \\
 \hline
 Training inputs & 60k & 60k \\
 \hline
 Hidden size & 1024 & 1024 \\
 \hline
 Nonlinear activation &  ReLU &  ReLU\\
 \hline
 Architecture & Fully Connected & Fully Connected \\
 \hline
 Posterior Samples & 500 & 500\\
 \hline
 Numerical Integrator Stepsize & 0.002 & 0.001\\
 \hline
 Number of steps for Numerical Integrator & 10 & 10\\
 \bottomrule
\end{tabular}
\label{tbl:bnns_hmc_1}}
\end{table}

\begin{table}[ht]
\centering
 {
\caption{
Hyperparameters for the BNN using VI of Figures \ref{fig:gradient-heatmaps} and \ref{fig:expLossGradients_stripplot_HMC}.}
\begin{tabular}{p{3cm}||p{2cm}|p{2cm}}
 \multicolumn{3}{c}{\textbf{Training hyperparameters for VI}}\\
 \toprule
 Dataset & MNIST & Fashion MNIST \\
 \hline
 Training inputs & 60k & 60k \\
 \hline
 Hidden size & 512 & 1024 \\
 \hline
 Nonlinear activation & Leaky ReLU & Leaky ReLU\\
 \hline
 Architecture & Convolutional & Convolutional \\
 \hline
 Training epochs & 5 & 10\\
 \hline
 Learning rate & 0.01 & 0.001 \\
 \bottomrule
\end{tabular}
\label{table:bnns_vi}
}
\end{table}

\begin{table}[ht]
\centering {
\caption{
Hyperparameters for training BNNs with HMC for results reported in Figure \ref{fig:bayesmnist}. * indicates the parameters used in Table \ref{tab:attackcomp}.}
\begin{tabular}{ p{3.5cm}||p{4cm} }
 \multicolumn{2}{c}{\textbf{HMC MNIST/Fashion MNIST grid search}}\\
 \toprule
 Posterior samples & \{250*, 500, 750\}\\
 \hline
 Numerical Integrator Stepsize & \{0.025*, 0.01, 0.005, 0.001, 0.0001\} \\
 \hline
 Numerical Integrator Steps & \{10, 15, 20*\} \\
 \hline
 Hidden size & \{128, 256, 512*\}\\
 \hline
 Nonlinear activation & \{relu*, tanh, sigmoid\}\\
 \hline
 Architecture & \{1*,2,3\} fully connected layers\\
 \bottomrule
\end{tabular}
\label{table:MNIST_hmc}}
\end{table}

\begin{table}[ht]
\centering {
\caption{
Hyperparameters for training BNNs with VI fore results reported in Figure \ref{fig:bayesmnist}. * indicates the parameters used in Table \ref{tab:attackcomp}.}
\begin{tabular}{ p{3.5cm}||p{4cm} }
 \multicolumn{2}{c}{\textbf{VI MNIST/Fashion MNIST grid search}}\\
 \toprule
 Learning Rate & \{0.001*\} \\
  \hline
 Minibatch Size & \{128*, 256, 512, 1024\} \\
 \hline
 Hidden size & \{64, 128, 256, 512*, 1024\}\\
 \hline
 Nonlinear activation & \{relu*, tanh, sigmoid\}\\
 \hline
 Architecture & \{1*,2,3\} fully connected layers\\
  \hline
 Training epochs & \{3,5,7,9,12,15*\} epochs\\
 \bottomrule
\end{tabular}
\label{table:MNIST_VI}}
\end{table}

\begin{table}[ht]
\centering  {
\caption{
Hyperparameters for training NNs with SGD for results reported in Figure \ref{fig:bayesmnist}. * indicates the parameters used in Table \ref{tab:attackcomp}.}
\begin{tabular}{ p{3.5cm}||p{4cm} }
 \multicolumn{2}{c}{\textbf{SGD MNIST/Fashion MNIST grid search}}\\
 \toprule
 Learning Rate & \{0.001, 0.005, 0.01, 0.05*\} \\
  \hline
 Hidden size & \{64, 128, 256, 512*\}\\
 \hline
 Nonlinear activation & \{relu*, tanh, sigmoid\}\\
 \hline
 Architecture & \{1*, 2, 3, 4, 5\} fully connected layers\\
  \hline
 Training epochs & \{5, 10, 15, 20*, 25\} epochs\\
 \bottomrule
\end{tabular}
\label{table:MNIST_SGD}}
\end{table}

\end{document}